\def\llncs{0}
\def\mnotes{0}

\ifnum\llncs=1
    \documentclass[envcountsect,runningheads]{llncs}
    \usepackage{pifont}
    \usepackage{color,amsfonts,amsmath,url,wrapfig}

\else

    \documentclass[11pt]{article}
    \usepackage{color,amsfonts,amsmath,url,wrapfig}

    \usepackage{fullpage,times,amsthm,latexsym,amssymb,graphicx}

\usepackage[colorlinks=true,breaklinks=true,linkcolor=black,
citecolor=black,filecolor=black,menucolor=black,
pagecolor=black,urlcolor=black]{hyperref}

\advance \topmargin by -\headheight

\advance \topmargin by -\headsep

\fi 

\ifnum\llncs=0

\newtheorem{theorem}{Theorem}[section]
\newtheorem{lemma}[theorem]{Lemma}
\newtheorem{mylemma}[theorem]{Lemma}

\newtheorem{proposition}[theorem]{Proposition}
\newtheorem{myproposition}[theorem]{Proposition}
\newtheorem{claim}[theorem]{Claim}
\newtheorem{myclaim}[theorem]{Claim}
\newtheorem{fact}[theorem]{Fact}
\newtheorem{corollary}[theorem]{Corollary}
\newtheorem{mycorollary}[theorem]{Corollary}

\newtheorem{definition}{Definition}[section]

\theoremstyle{definition}

\theoremstyle{remark}

\newtheorem{myremark}{Remark} [section]
\newenvironment{remark}{\begin{myremark}}{$\diamondsuit$\end{myremark}}

\newtheorem{myexample}{Example}

\else

\fi

\usepackage{amsfonts,amsmath, amsopn, amssymb, epsfig, graphicx, multirow, url, verbatim}
\usepackage[ruled,vlined,noline,noend]{algorithm2e}
\usepackage{color,bm}
\ifnum\mnotes=0
\newcommand{\mnote}[1]{}
\else
\newcounter{mynotes}
\newcommand{\mnote}[1]{\addtocounter{mynotes}{1}{{\bf !}}%
\marginpar{\scriptsize  {\arabic{mynotes}.\ {\sf \textcolor{red}{#1}}}}}
\fi

\newcommand{\snote}[1]{\mnote{S: #1}}

\newcommand{\gnote}[1]{\mnote{G: #1}}

\newenvironment{myproof}{\begin{proof}
}{\ifnum\llncs=1
{~}\qed
\fi
\end{proof}}
\newenvironment{proofof}[1]{\begin{myproof}[\ifnum\llncs=0 Proof \fi of {#1}]
}{\end{myproof}}

\newcommand{\ignore}[1]{}

\newcommand{\eps}{\epsilon}

\DeclareMathOperator{\loglog}{\log\log}

\newcommand{\etal}{{\it et~al.}}

\SetAlgoProcName{}

\newcommand{\up}[1]{#1^{\uparrow}}
\newcommand{\down}[1]{#1^{\downarrow}}
\newcommand{\downdec}[1]{#1^{\le \downarrow}}
\newcommand{\monlb}[1]{f^{mon}_{\down{#1}}}
\newcommand{\ind}[1]{\mathbf 1_{#1}}
\newcommand{\live}[1]{\text{live}(#1)}

\newcommand{\stars}{stars}
\newcommand{\dtd}{\mbox{DT-depth}}
\newcommand{\fdeg}{deg}
\newcommand{\myarg}{S}
\newcommand{\myvar}{Y}
\newcommand{\myrange}{r}
\newcommand{\mywidth}{k}
\newcommand{\mytrange}{R_r}
\newcommand{\mydnf}{\text{DNF}^{\mywidth,\myrange}}
\newcommand{\myidnf}{\text{DNF}^{\mywidth,\myrange}}
\newcommand{\myrand}[1]{#1|_{\rho}}
\newcommand{\myindexk}{t}
\newcommand{\mysize}{n}

\newcommand{\domain}{\ensuremath{2^{[\mysize]}}}
\newcommand{\range}{\ensuremath{\{0,\dots,\mywidth\}}}
\newcommand{\f}{\ensuremath{f:\domain\to\range}}

\begin{document}

\ifnum\llncs=1
    \bibliographystyle{splncs03}
\else
    \bibliographystyle{alpha}
\fi

\title{Learning Pseudo-Boolean $k$-DNF and Submodular Functions\thanks{This material is based upon work supported by NSF CAREER award CCF-0845701.}}

\ifnum\llncs=1

\author{
Sofya Raskhodnikova\inst{1},
and Grigory Yaroslavtsev\inst{1}
}

\institute{
Pennsylvania State University, USA. {\tt \{sofya, grigory\}@cse.psu.edu}.
}

\authorrunning{S. Raskhodnikova and G. Yaroslavtsev}
\titlerunning{Learning Pseudo-Boolean $k$-DNF and Submodular Functions}
\else
   \author{
        Sofya Raskhodnikova\thanks{Pennsylvania State University, USA. {\tt \{sofya, grigory\}@cse.psu.edu}.
        }
        \and Grigory Yaroslavtsev\protect\footnotemark[2]
   }
\fi

\maketitle

\begin{abstract}
We prove that any submodular function $f:\{0,1\}^n \rightarrow \{0,1,...,k\}$ can be represented as a pseudo-Boolean $2k$-DNF formula. Pseudo-Boolean DNFs are a natural generalization of DNF representation for functions with integer range. Each term in such a formula has an associated integral constant. We show that an analog of H{\aa}stad's switching lemma holds for pseudo-Boolean $k$-DNFs if all constants associated with the terms of the formula are bounded.

This allows us to generalize Mansour's PAC-learning\snote{Should talk about agnostic learning instead?} algorithm for $k$-DNFs to pseudo-Boolean $k$-DNFs, and hence gives a PAC-learning algorithm with membership queries under the uniform distribution for submodular functions of the form $f:\{0,1\}^n \rightarrow \{0,1,...,k\}$. Our algorithm
 runs in time polynomial in $n$, $\mywidth^{O(\mywidth \log \mywidth / \epsilon)}$, $1/\epsilon$ and $\log (1/\delta)$  and works even in the agnostic setting.
The line of previous work on learning submodular functions [Balcan, Harvey (STOC '11), Gupta, Hardt, Roth, Ullman (STOC '11), Cheraghchi, Klivans, Kothari, Lee (SODA '12)] implies only $n^{O(k)}$ query complexity for learning submodular functions in this setting, for fixed $\eps$ and $\delta$.

Our learning algorithm implies a property tester for submodularity of functions $f:\{0,1\}^n\to\range$ with query complexity polynomial in $n$ for $k=O((\log n/\loglog n)^{1/2})$ and constant proximity parameter $\eps$.
\end{abstract}

\thispagestyle{empty}
\setcounter{page}{0}
\newpage
\section{Introduction}\label{sec:intro}
We investigate learning of submodular set functions, defined on the ground set $[n]=\{1,\dots,n\}$. A set function $f: \domain\to\mathbb{R}$ is {\em submodular} if
one of the following equivalent definitions holds:
\begin{enumerate}
\item\label{def:submodular1} $f(S) + f(T) \ge f(S \cup T) + f(S \cap T) \text{ for all } S,T\subseteq[n].$
\item\label{def:submodular2} $f(S \cup \{i\})-f(S)\ge f(T \cup \{i\})-f(T)  \text{ for all } S\subset T \subseteq [n] \text{ and } i\in [n]\setminus T$.
\item\label{def:submodular3} $f(S \cup \{i\}) + f(S \cup \{j\}) \ge f(S \cup \{i,j\}) + f(S) \text{ for all }  S\subseteq[n] \text{ and } i,j \in [n]\setminus S$.
\end{enumerate}
Submodular set functions are important and widely studied, with applications in combinatorial optimization, economics, algorithmic game theory and many other disciplines.
In many contexts, submodular functions are integral and nonnegative, and this is the setting we focus on. Examples of such functions include coverage functions\footnote{\label{fn:coverage-fn}Given sets $A_1,\dots,A_n$ in the universe $U$, a coverage function is $f(S)=|\cup_{j\in S} A_j|$.}, matroid rank functions, functions modeling valuations when the value of each set is expressed in dollars, cut functions of graphs\footnote{Given a graph $G$ on the vertex set $[n]$, the cut function $f(S)$  of $G$
is the number of edges of $G$ crossing the cut $(S, [n]/S)$).}, and cardinality-based set functions, i.e., functions of the form $f(S)=g(|S|)$, where $g$ is concave.

We study submodular functions $f: \domain\to \{0,1, \dots, k\}$, and give a learning algorithm for this class. To obtain our result, we use tools from several diverse areas, ranging from operations research to complexity theory. \snote{Added; please check/comment.}

\paragraph{Structural result.} The first ingredient in the design of our algorithm is a structural result which shows that every submodular function in this class can be represented by a narrow pseudo-Boolean disjunctive normal form (DNF) formula, which naturally generalizes DNF for pseudo-Boolean functions. Pseudo-Boolean DNF formulas are well studied.
(For an introduction to pseudo-Boolean functions and normal forms, see \S 13 of the book by Crama and Hammer~\cite{CramaH11}.)

In the next definition and the rest of the paper, we use domains $\domain$ and $\{0,1\}^n$ interchangeably. They are equivalent because there is a bijection between sets $S\subseteq [n]$ and strings $x_1\dots x_n\in\{0,1\}^n$, where the bit $x_i$ is mapped to 1 if $i\in S$ and to 0 otherwise.

\begin{definition}[Pseudo-Boolean DNF]\label{def:pb-dnf}
Let $x_1,\dots,x_n$ be variables taking values in $\{0,1\}$. A {\em pseudo-boolean DNF of width $k$ and size $s$} (also called a $k$-DNF of size $s$) is an expression of the form
$$f(x_1, \dots, x_n) =\max_{t = 1}^s \Big(a_t \Big(\bigwedge_{i\in A_t}x_i\Big)\Big(\bigwedge_{j\in B_t}\bar{x}_j\Big)\Big),$$
where $a_t$ are constants, $A_t,B_t\subseteq [n]$ and $|A_t|+|B_t|\leq k$ for $t\in[s]$.\snote{Restored from a previous version. Did this change by mistake?}
A pseudo-boolean DNF is {\em monotone} if it contains no negated variables, i.e., $B_t=\emptyset$ for all {\em terms} in the $\max$ expression.
The class of all functions that have pseudo-Boolean $k$-DNF representations with constants $a_t \in \{0, \dots r\}$ is denoted $\myidnf$.
\end{definition}
It is not hard to see that every set function $\f$ has a pseudo-Boolean DNF representation with constants $a_t\in\range$, but in general there is no bound on the width of the formula.

Our structural result, stated next, shows that every submodular function $\f$ can be represented by a pseudo-Boolean $2k$-DNF with constants $a_t\in\range$. Our result is stronger for {\em monotone} functions, i.e., functions satisfying $f(S)\leq f(T)$ for all $S\subseteq T\subseteq [n]$. Examples of monotone submodular functions include coverage functions and matroid rank functions.
\begin{theorem}[DNF representation of submodular functions]\label{thm:k-dnf-decomposition} \gnote{We build on ideas of Gupta et al. In the monotone case our decomposition can be obtained from their result, but for the general case we have to work harder. Key idea: monotone extension from Lovasz82.}
Each submodular function $f: \{0,1\}^n \to \range$ can be represented by a pseudo-Boolean $2k$-DNF with constants $a_t\in \range$ for all $t\in[s]$.
Moreover, each term of the pseudo-Boolean DNF has at most $k$ positive and at most $k$ negated variables, i.e., $|A_t|\leq k$ and $|B_t|\leq k$.\snote{Added. Increases potential usability of our result by others.}
If $f$ is monotone then its representation is a monotone pseudo-Boolean $k$-DNF.
\end{theorem}
Note that the converse of Theorem~\ref{thm:k-dnf-decomposition} is false. E.g., consider the function $f(S)$ that maps $S$ to 0 if $|S|\leq 1$ and to 1 otherwise. It can be represented by a 2-DNF as follows: $f(x_1\dots x_n)=\max_{i,j\in[n]} x_i\wedge x_j$. However, it is not submodular, since version \ref{def:submodular3} of the definition above is falsified with $S=\emptyset, i=1$ and $j=2$.

\snote{Hopefully less terrible, but can be improved.} Our proof of Theorem~\ref{thm:k-dnf-decomposition} builds on techniques developed by Gupta~\etal~\cite{GHRU11} who show how to decompose a given submodular function into Lipschitz submodular functions. We first prove our structural result for monotone submodular functions. We use the decomposition from \cite{GHRU11} to cover the domain of such a function by regions where the function is constant and then capture each region by a monotone term of width at most $k$.  Then we decompose a general submodular function $f$ into monotone regions, as in \cite{GHRU11}. For each such region, we construct a monotone function which coincides with $f$ on that region, does not exceed $f$ everywhere else, and can be represented as a narrow pseudo-Boolean $k$-DNF by invoking our structural result for monotone submodular functions. This construction uses a monotone extension of submodular functions defined by Lovasz~\cite{L82}.

\paragraph{Learning.}
Our main result is a PAC-learning algorithm with membership queries under the uniform distribution for pseudo-Boolean $k$-DNF, which by Theorem~\ref{thm:k-dnf-decomposition} also applies to submodular functions $\f$. We use a (standard) variant of the PAC-learning definition given by Valiant \cite{Val84}.

\begin{definition}[PAC and agnostic learning under uniform distribution]
Let $U^n$ be the uniform distribution on $\{0,1\}^n$.
A class of functions $\mathcal C$ is {\em PAC-learnable} under the uniform distribution if there exists a randomized algorithm $\mathcal A$, called a {\em PAC-learner}, which for every function $f \in \mathcal C$ and every $\epsilon, \delta >0$, with probability at least $1 - \delta$ over the randomness of $\mathcal A$,  outputs a hypothesis $h$, such that
\begin{eqnarray}\label{eq:PAC-guarantee}
\Pr_{x \sim U^n}[h(x) \neq f(x)] \le \epsilon.
\end{eqnarray}
A learning algorithm $\mathcal A$ is {\em proper} if it always outputs a hypothesis $h$ from the class ${\mathcal C}$. A learning algorithm is {\em agnostic} if it works even if the input function $f$ is arbitrary (not necessarily from $\mathcal C$), with $\eps$ replaced by $opt+\eps$ in (\ref{eq:PAC-guarantee}), where $opt$ is the smallest achievable error for a hypothesis in $\cal C$.
\end{definition}

\noindent Our algorithm accesses its input $f$ via {\em membership queries}, i.e., by requesting $f(x)$ on some $x$ in $f$'s domain.
\begin{theorem}\label{thm:learning-pb-dnf}
The class of pseudo-Boolean $k$-DNF formulas on $n$ variables with constants in the range $\{0, \dots, \myrange\}$ is PAC-learnable with membership queries under the uniform distribution with running time polynomial in $n$, $\mywidth^{O(\mywidth \log \myrange / \epsilon)}$, $1/\epsilon$ and $\log (1/\delta)$, even in the agnostic setting.
\end{theorem}

\snote{Please check.}Our (non-agnostic) learning algorithm is a generalization of Mansour's PAC-learner for $k$-DNF~\cite{M95}.
It consists of running the algorithm of Kushilevitz and Mansour~\cite{KM91} for learning functions that can be approximated by functions with few non-zero Fourier coefficients, and thus has the same running time (and the same low-degree polynomial dependence on $n$). To be able to use this algorithm, we prove (in Lemma~\ref{lem:pb-dnf-approximation}) that all functions in $\mydnf$ have this property. The agnostic version of our algorithm follows from the Fourier concentration result in Lemma~\ref{lem:pb-dnf-approximation} and the work of Gopalan, Kalai and Klivans~\cite{GKK08}.

The key ingredient in the proof of Lemma~\ref{lem:pb-dnf-approximation} (on Fourier concentration) is a generalization of H{\aa}stad's switching lemma~\cite{H86,B94} for standard DNF formulas to pseudo-Boolean DNF. Our generalization (formally stated in Lemma~\ref{lem:gen-switching-lemma}) asserts that a function $f \in \myidnf$, restricted on large random subset of variables to random Boolean values, with large probability can be computed by a decision tree of small depth. (See Section~\ref{sec:switching-lemma} for definitions of random restrictions and decision trees.) Crucially, our bound on the probability that a random restriction of $f$ has large decision-tree complexity is only a factor of $r$ larger than the corresponding guarantee for the Boolean case.\gnote{Say that our main insight is that one can lose only a multiplicative factor in range in the guarantee.}\snote{This sentence is supposed to convey that. ("Insight" means "idea", this seems more like a result. Maybe you can identify "insights" needed to get this dependence.)}

Theorems~\ref{thm:learning-pb-dnf} and~\ref{thm:k-dnf-decomposition} imply the following corollary.
\begin{corollary}\label{cor:learning-submodular}
The class of submodular functions $f \colon \{0,1\}^n \rightarrow \range$ is PAC-learnable with membership queries under the uniform distribution in time polynomial in $n$, $k^{O(k \log {k / \epsilon})}$ and $\log (1/\delta)$.\snote{Even in the agnostic setting.}
\end{corollary}

\paragraph{Implications for testing submodularity.} Our results give property testers for submodularity of functions $\f$. A {\em property tester} \cite{RS96,GGR98}
 is given oracle access to an object  and a proximity parameter $\eps \in (0,1)$. If the object has the desired property, the tester \emph{accepts} it with probability at least $2/3$; if the object is $\eps$-far from having the desired property then the tester \emph{rejects} it with probability at least $2/3$. Specifically, for properties of functions, $\eps$-far means that a given function differs on at least an $\eps$ fraction of the domain points from any function with the property.

 As we observe in Proposition~\ref{prop:lerner-to-proper-learner}, a learner for a discrete class (e.g., the class of functions $\f$) can be converted to a proper learner with the same query complexity (but huge overhead in running time). Thus, Corollary~\ref{cor:learning-submodular} implies a tester for submodularity of functions $\f$  with query complexity polynomial in $n$ and $k^{O(k \log {k / \epsilon})}$,
making progress on a question posed by Seshadhri \cite{Sesh11}.

%



\subsection{Related work}\label{previous-work}
\paragraph{Structural results for Boolean submodular functions.} For the special case of Boolean functions, characterizations of submodular and monotone submodular functions in terms of simple DNF formulas are known. A Boolean function is monotone submodular if and only if it can be represented as a monotone 1-DNF (see, e.g., Appendix A in~\cite{BH11}).
A Boolean function is submodular if and only if it has a pure (without singleton terms) 2-DNF representation \cite{EHP97}.

\paragraph{Learning submodular functions.}
The problem of learning submodular functions has recently attracted significant interest. The focus on learning-style guarantees, which allow one to make arbitrary mistakes on some small portion of the domain, is justified by a negative results of Goemans~\etal~\cite{GoemansHIM09}. It demonstrates
that every algorithm that makes a polynomial in $n$ number of queries to a monotone submodular function (more specifically, even a matroid rank function) and tries to approximate it on all points in the domain, must make an $\Omega(\sqrt{n}/\log n)$ multiplicative error on some point.

Using results on concentration of Lipschitz submodular functions~\cite{BCM00, BCM09, V11} and on noise-stability of submodular functions~\cite{CheraghchiKKL12}, significant progress on learning submodular functions was obtained by Balcan and Harvey~\cite{BH11,BH11a}, Gupta~\etal~\cite{GHRU11} and Cheraghchi~\etal~\cite{CheraghchiKKL12}.
These works obtain learners that {\em approximate} submodular functions, as opposed to learning them exactly, on an $\epsilon$ fraction of values in the domain. However, their learning algorithms generally work with weaker access models and with submodular functions over more general ranges.

Balcan and Harvey's algorithms learn a function within a given {\em multiplicative error} on all but $\eps$ fraction of the probability mass (according to a specified distribution on the domain). Their first algorithm learns {\em monotone} nonnegative submodular functions over $\domain$ within a {\em multiplicative factor of $\sqrt{n}$} over {\em arbitrary} distributions using only {\em random examples} in polynomial time. For the special case of {\em product distributions} and {\em monotone} nonnegative submodular functions {\em with Lipschitz constant 1}, their second algorithm can learn {\em within a constant factor} in polynomial time.

Gupta~\etal~\cite{GHRU11} design an algorithm that learns a submodular function with the range $[0,1]$ within a given {\em additive error $\alpha$} on all but $\eps$ fraction of the probability mass (according to a specified {\em product distribution} on the domain). Their algorithm requires membership queries, but works {\em even when these queries are answered with additive error $\alpha/4$.} It takes $n^{O(\log (1/\eps)/ \alpha^2)}$ time.

Cheraghchi~\etal~\cite{CheraghchiKKL12} also work with {\em additive error}. Their learner is agnostic and only uses {\em statistical queries}. It produces a hypothesis which (with probability at least $1-\delta$) has the expected additive error $opt+\alpha$ with respect to a {\em product distribution}, where $opt$ is the error of the best concept in the class. Their algorithm runs in time polynomial in $n^{O(1/\alpha)}$ and $\log (1/\delta)$.

Observe that the results in \cite{GHRU11} and \cite{CheraghchiKKL12} directly imply an $n^{O(\log (1/\eps) k^2)}$ time algorithm for our setting, by rescaling our input function to be in $[0,1]$ and setting the error $\alpha=1/(2r).$ The techniques in \cite{GHRU11} also imply $n^{O(k)}$ time complexity for non-agnostically\snote{Check this; I am not sure how to get this claimed dependence} learning submodular functions in this setting, for fixed $\eps$ and $\delta$. To the best of our knowledge, this is the best dependence on $n$, one can obtain from previous work.

Finally, Chakrabarty and Huang \cite{ChaH12} gave an exact learning algorithm for coverage functions, a subclass of monotone submodular functions. Their algorithm makes $O(n|U|)$ queries, where $U$ is the size of the universe. (Coverage functions are defined as in Footnote~\ref{fn:coverage-fn} with additional nonnegative weight for each set, and $f(S)$ equal to the weight of $\cup_{j\in S} A_j$ instead of the cardinality.)

\paragraph{Property testing submodular functions.}
The study of submodularity in the context of property testing was initiated by Parnas, Ron and Rubinfeld~\cite{PRR03}.
Seshadhri and Vondrak~\cite{SV11} gave the first sublinear (in the size of the domain) tester for submodularity of set functions. Their tester works for all ranges and has query and time complexity  $\left(1/\epsilon\right)^{O(\sqrt{n}\log n)}$.
They also showed a reduction from testing monotonicity 
to testing submodularity which, together with a lower bound for testing monotonicity given by
Blais, Brody and Matulef~\cite{BBM11}, implies a lower bound of $\Omega(n)$ on the query complexity of testing submodularity for an arbitrary range and constant $\epsilon > 0$.

Given the large gap between known upper and lower bounds on the complexity of testing submodularity, Seshadhri \cite{Sesh11} asked for testers for several important subclasses of submodular functions. The exact learner of Chakrabarty and Huang \cite{ChaH12} for coverage functions, mentioned above, gives a property tester for this class with the same query complexity.

For the special case of Boolean functions, in the light of the structural results mentioned above, one can test if a function is monotone submodular with $O(1/\epsilon)$ queries by using the algorithm from~\cite{PRS02} (Section 4.3) for testing monotone monomials.

\section{Representing submodular functions as pseudo-Boolean DNFs}\label{sec:submodular-to-DNF}
In this section, we prove Theorem~\ref{thm:k-dnf-decomposition} that shows that every submodular function over a bounded (nonnegative) integral range can be represented by a narrow pseudo-Boolean DNF. After introducing notation used in the rest of the section (in Definition~\ref{def:down-up}), we prove the theorem for the special case when $f$ is monotone submodular (restated in Lemma~\ref{lem:k-dnf-decomposition-monotone}) and then present the proof for the general case. \snote{Change, please check}In the proof, we give a recursive algorithm for constructing pseudo-Boolean DNF representation which has the same structure of recursive calls as the decomposition algorithm of Gupta~\etal~\cite{GHRU11}. Our contribution is in showing how to use these calls to get a pseudo-Boolean $2k$-DNF representation of the input function.

\begin{definition}[$\down{S}$ and $\up{S}$]\label{def:down-up}
For a set $S \in \domain$, we denote the collection of all subsets of $S$ by $\down{S}$ and the collection of all supersets of $S$ by $\up{S}$.
\end{definition}

\begin{lemma}[DNF representation of monotone submodular functions]\label{lem:k-dnf-decomposition-monotone}
Each monotone submodular function $f: \{0,1\}^n \to \range$ can be represented by a pseudo-Boolean monotone $2k$-DNF with constants $a_t\in \range$ for all $t\in[s]$.
\end{lemma}

\begin{proof}
Algorithm~\ref{alg:monotone-dnf} below, with the initial call \textsc{Monotone-DNF}$(f,\emptyset)$, returns the collection of terms in a pseudo-boolean DNF representation of $f$.

\begin{algorithm}
\SetKwInOut{Input}{input}\SetKwInOut{Output}{output}
\Input{Oracle access to $\f$, argument $\myarg \in \domain$.} \Output{Collection $C$ of monotone terms of width at most $k$.}
\DontPrintSemicolon
\BlankLine
\nl $C \leftarrow (f(\myarg) \cdot \bigwedge\limits_{i \in \myarg} x_i)$\;
\nl \For{$j \in [n]\setminus \myarg$} {
\nl	\If {$f(\myarg \cup \{j\}) > f(\myarg)$} {
\nl 	$C \leftarrow C \cup $ \textsc{Monotone-DNF}$(f, \myarg \cup \{j\})$.}}
\nl \Return $C$ \;
\caption{\textsc{Monotone-DNF}$(f, \myarg)$.\label{alg:monotone-dnf}}
\end{algorithm}

First, note that the invariant $f(\myarg) \ge |\myarg|$ is maintained for every call \textsc{Monotone-DNF}$(f,\myarg)$. Since the maximum value of $f$ is at most $k$, there are no calls with $|\myarg| > k$. Thus, every term in the collection returned by \textsc{Monotone-DNF}$(f, \emptyset)$ has width at most $k$. By definition, all terms are monotone.

Next, we show that the resulting formula $\max\limits_{C_i \in C} C_i$ exactly represents $f$.
For all $\myvar \in \domain$ we have $f(\myvar) \ge \max\limits_{C_i \in C} C_i(\myvar)$ by monotonicity of $f$.
To see that $f(\myvar) \le \max\limits_{C_i \in C} C_i(\myvar)$
let $\mathcal T = \{Z \mid Z \subseteq \myvar, f(Z) = f(\myvar)\}$ and $T$ be a set of the smallest size in $\mathcal T$.
If there was a recursive call \textsc{Monotone-DNF}$(f,T)$ then the term added by this recursive call would ensure the inequality.
If $T = \emptyset$ then such a call was made.
Otherwise, consider the set ${\mathcal U} = \{T \setminus \{j\} \mid j\in T\}$.
By the choice of $T$, we have $f(Z) < f(T)$ for all $Z \in {\mathcal U}$.
By submodularity of $f$, this implies that the restriction of $f$ on $\down{T}$ is a strictly increasing function. Thus, the recursive call \textsc{Monotone-DNF}$(f,T)$ was made and the term added by it guarantees the inequality.
\end{proof}

For a collection $\mathcal S$ of subsets of $[n]$, let $f_{\mathcal S} \colon \mathcal S \rightarrow \mathbb R$ denote the restriction of a function $f$ to the union of sets in $\mathcal S$.
We use notation $\ind{\mathcal S} \colon \domain \rightarrow \{0,1\}$ for the indicator function defined by $\ind{\mathcal S}(Y) = 1$ iff $Y \in \bigcup\limits_{S\in\mathcal S}S$.

\begin{proofof}{Theorem~\ref{thm:k-dnf-decomposition}}
For a general submodular function, the formula can be constructed using Algorithm~\ref{alg:dnf}, with the initial call \textsc{DNF}$(f, [n])$. The algorithm description uses the function $\monlb{S}$, defined next.
\begin{definition}[Function $\monlb{S}$]
For a set $S\subseteq [n]$, define the function $\monlb{S} \colon \down{S} \rightarrow \range$ as follows: $\monlb{S}(Y) = \min_{Y \subseteq Z \subseteq S} f(Z)$.
\end{definition}
\begin{proposition}[Proposition 2.1 in \cite{L82}]\label{prop:monotone-lb}\gnote{Rephrase for arbitrary function.}
If $f_{\down{S}}$ is a submodular function, then $\monlb{S}$ is a monotone submodular function.
\end{proposition}

\begin{algorithm}
\SetKwInOut{Input}{input}\SetKwInOut{Output}{output}
\Input{Oracle access to $\f$, argument $\myarg \in \domain$.} \Output{Collection $C$ of terms, each containing at most $k$ positive and at most $k$ negated variables.}
\DontPrintSemicolon
\BlankLine
\nl $C_{mon} \leftarrow $\textsc{Monotone-DNF}$(\monlb{\myarg}, \emptyset)$\; \label{ln:monotone-call}
\nl $C \leftarrow \bigcup\limits_{C_i\in C_{mon}}(C_i\cdot(\bigwedge\limits_{i \in [n]\setminus\myarg} \bar{x}_i))$\; \label{ln:multiply-terms}
\nl \For{$j \in \myarg$} {
\nl	\If {$f(\myarg \setminus \{j\}) > f(\myarg)$} {
\nl 	$C \leftarrow C \cup $ \textsc{DNF}$(f, \myarg \setminus \{j\})$.}}
\nl \Return $C$ \;
\caption{\textsc{DNF}$(f, \myarg)$.\label{alg:dnf}}
\end{algorithm}

Let $\mathcal S$ be the collection of sets $\myarg \subseteq [n]$
for which a recursive call is made when \textsc{DNF}$(f, [n])$ is executed.
For a set $S \in \mathcal S$, let $B(S) = \{j \mid f(S \setminus \{j\}) \leq f(S)\}$ be the set consisting of elements such that if we remove them from $S$, the value of the function does not increase.
Let the {\em monotone region} of $S$ be defined by $\downdec{S} = \{Z \mid (S \setminus B(S)) \subseteq Z \subseteq S\} = \down{S} \cap \up{(S \setminus B(S))}$. By submodularity of $f,$ the restriction $f_{\downdec{S}}$ is a monotone nondecreasing function.\snote{Need more explanations.}

\begin{proposition}\label{prop:lb-cover}
Fix $S \in \mathcal S$. Then  $f(Y) \ge \monlb{S}(Y)$ for all $Y \in \down{S}$.
Moreover,  $f(Y) = \monlb{S}(Y)$ for all $Y \in \downdec{S}$.
\end{proposition}
\begin{proof}
By the definition of $\monlb{S}$, we have
$\monlb{S}(Y) = \min_{Y \subseteq Z \subseteq S} f(Z) \le f(Y)$ for all $Y \in \down{S}$.
Since the restriction $f_{\downdec{S}}$ is monotone nondecreasing,  $\monlb{S}(Y) = \min_{Y \subseteq Z \subseteq S} f(Z) = f(Y)$ for all $Y \in \downdec{S}$.
\end{proof}

The following proposition is implicit in~\cite{GHRU11}. For completeness, we prove it in Appendix~\ref{app:omitted-proofs}.

\begin{proposition}\label{prop:monotone-cover}
For 
all functions $\f$, the collection of all monotone regions of sets in $\mathcal S$ forms a cover of the domain, namely, $\cup_{S \in \mathcal S} \downdec{S} = \domain$.
\end{proposition}

Lemma~\ref{lem:k-dnf-decomposition-monotone} and Proposition~\ref{prop:monotone-lb} give that
the collection of terms $C_{mon}$, constructed in Line~\ref{ln:monotone-call} of Algorithm~\ref{alg:dnf}, corresponds to a monotone pseudo-Boolean $k$-DNF representation for $\monlb{\myarg}$.
By the same argument as in the proof of Lemma~\ref{lem:k-dnf-decomposition-monotone}, $|\myarg| \geq n - k$ for all $\myarg \in \mathcal S$, since the maximum of $f$ is at most $k$. Therefore, Line~\ref{ln:multiply-terms}  of Algorithm~\ref{alg:dnf} adds at most $n - |\myarg|$ negated variables to every term of $C_{mon}$, resulting in terms with at most $k$ positive and at most $k$ negated variables.

It remains to prove that the constructed formula represents $f$. For a set $\myarg$, let $C_\myarg$ denote the collection of terms obtained on Line~\ref{ln:multiply-terms}  of Algorithm~\ref{alg:dnf}.  By construction, $C_\myarg(Y) = \monlb{\myarg} \cdot \ind{\down{\myarg}}(Y)$ for all $Y \in \domain$. For every $Y \in \domain$, the first part of Proposition~\ref{prop:lb-cover} implies that $C_\myarg(Y) = \monlb{\myarg} (Y) \cdot \ind{\down{\myarg}} (Y) \le f(Y)$, yielding $\max_{\myarg \in \mathcal S} C_\myarg (Y) \le f(Y)$.
On the other hand, by Proposition~\ref{prop:monotone-cover}, for every $Y \in \domain$ there exists a set $\myarg \in \mathcal S$,
such that $Y \in \downdec{\myarg}$. For such $\myarg$, the second part of Proposition~\ref{prop:lb-cover}
implies that $C_\myarg(Y) = \monlb{\myarg}(Y) \cdot \ind{\down{\myarg}} (Y) = f(Y)$. Therefore, $f$  is equivalent to $\max_{\myarg \in \mathcal S}C_\myarg $.
\end{proofof}


\section{Generalization of H{\aa}stad's switching lemma for pseudo-Boolean DNFs}\label{sec:switching-lemma}
The following definitions are stated for completeness and can be found in~\cite{O12,M95}.
\begin{definition}[Decision tree]\snote{If you took it directly from Ryan's book, say so: "The following standard definition is from \cite{O12}."}
A decision tree $T$ is a representation of a function $f \colon \{0,1\}^n \rightarrow \mathbb R$. It consists of a rooted binary tree in which the internal nodes are labeled by coordinates $i \in [n]$, the outgoing edges of each internal node are labeled 0 and 1, and the leaves are labeled by real numbers. We insist that no coordinate $i \in [n]$ appears more than once on any root-to-leaf path.

Each input $x \in \{0,1\}^n$ corresponds to a computation path in the tree $T$ from the root to a leaf. When the computation path reaches an internal node labeled by a coordinate $i \in [n]$, we say that $T$ queries $x_i$. The computation path then follows the outgoing edge labeled by $x_i$. The output of $T$ (and hence $f$) on input $x$ is the label of the leaf reached by the computation path. We identify a tree with the function it computes.
\end{definition}

The {\em depth} $s$ of a decision tree $T$ is the maximum length of any root-to-leaf path in $T$.
For a function $f$, $\dtd(f)$ is the minimum depth of a decision tree computing $f$.

\begin{definition}[Random restriction]
A restriction $\rho$ is a mapping of the input variables to $\{0, 1, \star\}$.
The function obtained from $f(x_1, \dots, x_n)$ by applying a restriction $\rho$ is denoted $\myrand{f}$.
The inputs of $\myrand{f}$ are those $x_i$ for which $\rho(x_i) = \star$ while all other variables are set according to $\rho$.
\end{definition}

A variable $x_i$ is {\em live} with respect to a restriction $\rho$ if $\rho(x_i)= \star$. The set of live variables with respect to $\rho$ is  denoted $\live{\rho}$.
A random restriction $\rho$ with parameter $p\in (0,1)$ is obtained by setting each $x_i$, independently, to $0,1$ or $\star$, so that
$\Pr[\rho(x_i) = \star] = p$ and $\Pr[\rho(x_i) = 1] = \Pr[\rho(x_i) = 0] = (1 - p)/2$.

We will prove the following generalization of the switching lemma~\cite{H86,B94}.

\begin{lemma}[Switching lemma for pseudo-Boolean formulas]\label{lem:gen-switching-lemma}
Let $f \in \myidnf$ and $\rho$ be a random restriction with parameter $p$ (i.e., $\Pr[\rho(x_i) = \star] = p$). Then\gnote{replace 5 by 7 everywhere}
$$\Pr[\dtd(\myrand{f}) \ge s] < r \cdot (7pk)^s.$$
\end{lemma}
\begin{proof}
We use the exposition of Razborov's proof of the switching lemma for Boolean functions, described in~\cite{B94}, as the basis of our proof and highlight\snote{Need to actually do it (more prominently than currently).} the modifications we made for non-Boolean functions.

Define $\mathcal R^{\ell}_n$ to be the set of all restrictions $\rho$
on a domain of $n$ variables that have exactly $\ell$ unset variables.
Fix some function $f \in \myidnf$, represented by a formula $F$, and assume that there is a total order on the terms of $F$ as well as on the indices of the variables.
A restriction $\rho$ is applied to $F$ in order, so that $F_\rho$ is a pseudo-Boolean DNF formula whose terms consist of those terms in $F$ that are not falsified by $\rho$, each shortened by removing any variables that are satisfied by $\rho$, and taken in the order of occurrences of the original terms on which they are based.

\begin{definition}[Canonical labeled decision tree]
The {\em canonical labeled decision tree} for $F$, denoted  $T(F),$ is defined inductively as follows:
\begin{enumerate}
\item If $F$ is a constant function then $T(F)$ consists of a single leaf node labeled by the appropriate constant.\gnote{Explain when this happens}
\item If the first term $C_1$ of $F$ is not empty then let $F'$ be the remainder of $F$ so that $F = \max (C_1, F')$.
Let $K$ be the set of variables appearing in $C_1$.
The tree $T(F)$ starts with a complete binary tree for $K$, which queries the variables in $K$ in the order of their indices.
Each leaf $v_{\sigma}$ in the tree is associated with a restriction $\sigma$
which sets the variables of $K$ according to the path from the root to $v_\sigma$.
For each $\sigma$, replace the leaf node $v_\sigma$ by the subtree $T(F_{\sigma})$\snote{Is $F_\sigma$ defined?}. For the unique assignment $\sigma$ satisfying $C_1$, also label the corresponding node by $L_\sigma$ equal to the maximum of the labels assigned to the predecessors of this node in the tree and the integer constant associated with the term $C_1$.
\end{enumerate}
\end{definition}

Note that for Boolean DNF formulas the internal nodes in the canonical labeled decision tree are never labeled.\snote{Not very clear that you are trying to communicate that you generalized something here.}
In this case, the definition above is equivalent to that in~\cite{B94}.
For pseudo-Boolean DNF formulas the label $L_\sigma$ of the internal node $\sigma$ represents that the value of the formula on the leaves in the subtree of $\sigma$ is at least $L_\sigma$.

Using the terminology introduced above, we can state the switching lemma as follows.\snote{By Chernoff bound, this lemma implies the switching lemma, but with some loss in constants. How do you manage not to loose anything in the constant on the RHS?}
\begin{lemma}
Let \snote{type mismatch}$F \in \myidnf$, $s \ge 0$, $p \le 1 / 7$ and $\ell = pn$. Then
$$\frac{|\{\rho \in \mathcal R^\ell_n \colon |T(\myrand{F})| \ge s\}|}{|\mathcal R^\ell_n|} < r(7 pk)^s.$$
\end{lemma}
\begin{proof}
Let $\stars(k,s)$ be the set of all sequences $\beta = (\beta_1, \dots, \beta_t)$  such that for each $j\in[t]$, the coordinate $\beta_j \in \{\star, -\}^k \setminus \{-\}^k$ and such that the total number of $\star$'s in all the $\beta_j$ is $s$.

Let $S \subseteq \mathcal R^{\ell}_n$ be the set of restrictions $\rho$ such that $|T(F|_\rho)| \ge s$. We will define an injective mapping from $S$ to the cartesian product $\mathcal R^{\ell - s}_n \times \stars(k,s) \times [2^s] \times [r]$\snote{Last coordinate could be 0}.

Let $F = \max_i C_i$.
Suppose that $\rho \in S$ and $\pi$\snote{Is $\pi$ a path or the corresponding restriction? Stick to one.} is the restriction associated with the lexicographically first path in $T(\myrand{F})$ of length at least $s$.
Trim the last variables in $\pi$ along the path from the root so that $|\pi| = s$. Let $c$ be the maximum label of the node on $\pi$ (or zero, if none of the  nodes on $\pi$ are labeled).
Partition the set of terms of $F$ into two sets $F'$ and $F''$, where $F'$ contains all terms with constants $> c$ and $F''$ contains all terms with constants $\le c$ (for Boolean formulas, $c = 0$ and $F = F'$).
We will use the subformula $F'$ and $\pi$ to determine the image of $\rho$.
The image of $\rho$ is defined by following the path $\pi$ in the canonical labeled decision tree for $F_\rho$ and using the structure of the tree.

Let $C_{\nu_1}$ be the first term of $F'$ that is not set to $0$ by $\rho$.
Since $|\pi| > 0$, such a term must exist and will not be an empty term (otherwise, the value of $F|_{\rho}$ is fixed to be $>c$).
Let $K$ be the set of variables in $C_{\nu_1}|\rho$ and let $\sigma_1$ be the unique restriction of the variables in $K$ that satisfies $C_{\nu_1}|\rho$.
Let $\pi_1$ be the part of $\pi$ that sets the variables in $K$.
We have two cases based on whether $\pi_1 = \pi$.

\begin{enumerate}
\item If $\pi_1 \neq \pi$ then by the construction of $\pi$, restriction $\pi_1$ sets all the variables in $K$. Note that the restriction $\rho \sigma_1$ satisfies the term $C_{\nu_1}$ but since $\pi_1 \neq \pi$ the restriction $\rho \pi_1$ does not satisfy term $C_{\nu_1}$.
\item If $\pi_1 = \pi$ then it is possible that $\pi$ does not set all of the variables in $K$. In this case we shorten $\sigma_1$ to the variables in $K$
that appear in $\pi_1$.
\end{enumerate}

Define $\beta_1 \in \{\star, -\}^{\mywidth}$ based on the fixed ordering of the variables in the term $C_{\nu_1}$ by letting the $j$th component of $\beta_1$ be $\star$ if and only if the $j$th variable in $C_{\nu_1}$ is set by $\sigma_1$.
Since $C_{\nu_1}|_\rho$ is not the empty term, $\beta_1$ has at least one $\star$. From $C_{\nu_1}$ and $\beta_1$ we can reconstruct $\sigma_1$.

Now by the definition of $T(\myrand{F})$, the restriction $\pi \setminus \pi_1$ labels a path in the canonical labeled decision tree $T(F|_{\rho \pi_1})$.
If $\pi_1 \neq \pi,$ we repeat the argument above, replacing $\pi$ and $\rho$ with $\pi \setminus \pi_1$  and $\rho \pi_1$, respectively, and find a term $C_{\nu_2}$ which is the first term of $F'$ not set to $0$ by $\rho \pi_1$.
Based on this, we generate $\pi_2, \sigma_2$ and $\beta_2$, as before.
We repeat this process until the round $t$ in which $\pi_1 \pi_2 \dots \pi_t = \pi$.

Let $\sigma = \sigma_1 \sigma_2 \dots \sigma_t$.
We define\snote{$\delta$ is already used, pick another letter} $\delta \in \{0,1\}^s$ to be a vector that indicates for each variable set by $\pi$ whether it is set to the same value as $\sigma$ sets it.
We define the image of $\rho$ in the injective mapping as a quadruple,
$\langle \rho \sigma_1 \dots \sigma_t, (\beta_1, \dots, \beta_t), \delta, c \rangle$.
Because $\rho \sigma \in \mathcal R^{\ell - s}_n$ and $(\beta_1, \dots, \beta_t) \in \stars(k,s)$ the mapping is as described above.

It remains to show that the defined mapping is indeed injective.
We will show how to invert it by reconstructing $\rho$ from its image.
We use $c$ to construct $F'$ from $F$.
The reconstruction procedure is iterative.
In one stage of the reconstruction we recover $\pi_1 \dots \pi_{i_1}, \sigma_1 \dots \sigma_{i - 1}$ and construct $\rho \pi_1 \dots \pi_{i - 1} \sigma_i \dots \sigma_ t$.
Recall that for $i < t$ the restriction $\rho \pi_1 \dots \pi_{i - 1} \sigma_i$ satisfies the term $C_{\nu_i}$, but does not satisfy terms $C_j$ for all $j <\nu_i$.
This holds if we extend the restriction by appending $\sigma_{i + 1} \dots \sigma_t$.
\gnote{Discuss the case $i = t$.}
Thus, we can recover $\nu_i$ as the index of the first term of $F'$
that is not falsified by $\rho \pi_1 \dots \pi_{i - 1} \sigma_i \dots \sigma_t$ and the consant corresponding to this term is at least $c$.

Now, based on $C_{\nu_1}$ and $\beta_i$, we can determine $\sigma_i$.
Since we know $\sigma_1 \dots \sigma_i$, using the vector $\delta$
we can determine $\pi_i$.
We can now change $\rho \pi_1 \dots \pi_{i - 1} \sigma_i \dots \sigma_t$ to $\rho \pi_1 \dots \pi_{i - 1} \pi_i \sigma_{i + 1} \dots \sigma_{t}$ using the knowledge of $\pi_i$ and $\sigma_i$.
Finally, given all the values of the $\pi_i$ we reconstruct $\rho$ by removing the variables from $\pi_1 \dots \pi_t$ from the restriction.

The following computation completes the proof and is given in Appendix~\ref{app:omitted-proofs} for completeness.
\begin{claim}[\cite{B94}]\label{clm:switching-lemma-computation}
For $p < 1/7$ and $p = \ell / n$ it holds that:
$$\frac{|\mathcal R^{\ell - s}_{n}| \cdot |\stars(k,s)| \cdot 2^s}{|\mathcal R^{\ell}_n|} < (7pk)^s$$
\end{claim}
\end{proof}
\end{proof}


\section{Learning pseudo-Boolean DNFs}\label{sec:learning}

In this section, we present our learning results for pseudo-Boolean $k$-DNF and prove Theorem~\ref{thm:learning-pb-dnf}.

Let $\mytrange$ denote the set of multiples of $2/(\myrange - 1)$ in the interval $[-1,1]$, namely $\mytrange = \{-1,-1+2/(\myrange - 1),...,1-2/(\myrange - 1),1\}$.
First, we apply a transformation of the range by mapping $\{0, \dots, \myrange\}$ to $\mytrange$.
Formally, in this section instead of functions $f \colon \{0,1\}^\mysize \rightarrow \{0, \dots, \myrange\}$  we consider functions $f' \colon \{-1,1\}^d \rightarrow [-1,1]$, such that $f'(x'_1, \dots, x'_n) = 2 / (\myrange - 1) \cdot f(x_1, \dots, x_n)  - 1$, where $x'_i = 2 x_i - 1$\snote{Usually, 0 is mapped to 1 and 1 to -1. Anyway, do we really need to transform the domain or transforming the range is enough?}.
\gnote{Requires discussion about additive constant.}
Note that a learning algorithm for the class of functions that can be represented by pseudo-Boolean DNF formulas of width $\mywidth$ with constants in the range $\mytrange$ implies Theorem~\ref{thm:learning-pb-dnf}.
Thus, to simplify the presentation we will abuse notation and refer to this class as $\mydnf$.\snote{Should have a different notation for this class. People who skip this paragraph, should not be mislead.}

For a set $S \subseteq [n]$, let $\chi_{S}$ be the standard Fourier basis vector and let $\hat f(S)$ denote the corresponding Fourier coefficient of a function $f$.

\begin{definition}
A function $g$ $\epsilon$-approximates a function $f$ if $\mathbb E[(f - g)^2] \le \epsilon$.
A function is $M$-sparse if it has at most $M$ non-zero Fourier coefficients. The Fourier degree of a function, denoted $\fdeg(f)$,
is the size of the largest set, such that $\hat f(S) \neq 0$.
\end{definition}

The following guarantee about approximation of functions in $\mydnf$ by sparse functions is the key lemma in the proof of Theorem~\ref{thm:learning-pb-dnf}.

\snote{This should be a theorem (too many lemmas (including the switching lemma), proof structure does not come out clearly)}\begin{lemma}\label{lem:pb-dnf-approximation}
Every function $f \in \mydnf$ can be $\epsilon$-approximated by an $M$-sparse function, where $M = \mywidth^{O(\mywidth \log (\myrange / \epsilon))}$.

\end{lemma}
\begin{proofof}{Lemma~\ref{lem:pb-dnf-approximation}}
We generalize the proof by Mansour~\cite{M95}, which relies on multiple applications of the switching lemma.
Our generalization of the switching lemma allows us to obtain the following parameters of the key statements in the proof, which bound the $L_2$-norm of the Fourier coefficients of large sets and the $L_1$-norm of the Fourier coefficients of small sets.
\begin{lemma}\label{lem:l2-bound}
For every function $f \in \mydnf$,
$$\sum_{S \colon |S| > 28 k \log (2r/\epsilon)} \hat f^2(S) \le \epsilon / 2.$$
\end{lemma}

\begin{lemma}\label{lem:l1-bound}
For every function $f \in \mydnf$\snote{and $\tau\in[n]$?},
$$\sum_{S \colon |S| \le \tau} |\hat f(S)| \le 4 r (28 k)^\tau = r k^{O(\tau)}.$$
\end{lemma}

Lemmas~\ref{lem:l2-bound} and~\ref{lem:l1-bound} are proved in Appendix~\ref{app:fourier-analysis}.

Let $\tau = 28 k \log (2r / \epsilon)$ and  $L = \sum_{|S| \le \tau} |\hat f(S)|$.
Let $G = \{S\ \colon |\hat f(S)| \ge \epsilon/ 2L \text{ and } |S| \le \tau \}$ and
$g(x) = \sum_{S \in G} \hat f(S) \chi_S(x)$. We will show that $g$ is $M$-sparse and that it $\eps$-approximates $f$.

By an averaging argument, $|G| \le 2 L^2 / \epsilon$. Thus, function $g$ is $M$-sparse, where $M \le 2 L^2 / \epsilon$. By Lemma~\ref{lem:l1-bound}, $L = r k^{O(\tau)} = k^{O(k \log(r / \epsilon))}$. Thus, $M = k^{O(k \log (r / \epsilon))}$, as claimed in the theorem statement.

By the definition of $g$ and by Parseval's identity,
$$\mathbb E[(f - g)^2] = \sum_{S \notin G} \hat f^2(S) = \sum_{S \colon |S| > \tau} \hat f^2(S) +\sum_{S \colon |S| \le \tau, |\hat f(S)| \le \epsilon / 2L } \hat f^2(S).$$
By Lemma~\ref{lem:l2-bound}, the first summation is at most $\epsilon / 2$.
For the second summation, we get:
$$\sum_{S \colon |S| \le \tau,  |\hat f(S)| \le \epsilon / 2L} \hat f^2(S)
\le \left(\max_{S \colon |\hat f(S)| \le \epsilon / 2L } |\hat f(S)|\right) \left(\sum_{|S| \le \tau} |\hat f(S)|\right) \le \frac{\epsilon}{2L} \cdot L = \epsilon / 2.$$
This implies that $\mathbb E[(f - g)^2] \le \epsilon$ and thus $g$ $\epsilon$-approximates $f$.
\end{proofof}

To get a learning algorithm and prove Theorem~\ref{thm:learning-pb-dnf} we can use the sparse approximation guarantee of Lemma~\ref{lem:pb-dnf-approximation}  together with Kushilevitz-Mansour learning algorithm (for PAC-learning) or the learning algorithm of Gopalan, Kalai and Klivans (for agnostic learning).

\begin{proofof}{Theorem~\ref{thm:learning-pb-dnf}}
We will use the learning algorithm of Kushilevitz and Mansour~\cite{GL89,KM91}, which gives the following guarantee:
\begin{theorem}[\cite{KM91}]\label{thm:km-learning}
Let $f$ be a function that can be $\eps$-approximated by an $M$-sparse function.
There exists a randomized algorithm, whose running time is polynomial
in $M$, $n$, $1 / \epsilon$ and $\log (1 / \delta)$, that given oracle access to $f$
 and $\delta > 0$, with probability at least $1 - \delta$ outputs a function $h$\snote{Use consistent notation in this theorem and the proposition that follows it.} that $O(\epsilon)$-approximates $f$.
\end{theorem}
Setting $M = \mywidth^{O(\mywidth \log (\myrange / \epsilon))}$\snote{Should be in terms of $\eps'$?} and the approximation parameter $\epsilon$ in Theorem~\ref{thm:km-learning} to be $\epsilon = \epsilon' / C \myrange^2$ for large enough constant $C$ we get an algorithm which returns a functions $h$ that  $(\epsilon' / \myrange^2)$-approximates $f$.
The running time of such algorithm is polynomial in $n$, $\mywidth^{O(\mywidth \log (\myrange/\epsilon'))}$ and $\log(1/\delta)$.
By Proposition~\ref{prop:rounding}, if we round the values of $h$ in every point to the nearest multiple of $2 /(\myrange - 1)$, we will get a function $h'$, such that $\Pr_{x \in U^n}[h'(x) \neq f(x)] \le \epsilon$, completing the proof.

\begin{proposition}\label{prop:rounding}
Suppose a function $g : \domain \to [-1,1]$ is an $\eps$-approximation for  $f \colon \domain \rightarrow \mytrange$. Let $h$ be the function defined by $h(x) = argmin_{y \in R_r} |g(x) - y|$\snote{Ambiguous for values in the middle.}. Then $\Pr_{x \in U^n}[h(x) \neq f(x)] \le \epsilon \cdot (r - 1)^2$.
\end{proposition}
\begin{proofof}{Proposition~\ref{prop:rounding}}
Observe that $|f(x) - g(x)|^2 \geq 1 / (r - 1)^2$ whenever $f(x) \neq h(x)$. This implies
\begin{eqnarray*}
\Pr_{x \in U^n}[h(x) \neq f(x)]
&\le& \Pr_{x \in U^n} [(r - 1)^2 \cdot |f(x) - g(x)|^2 \geq 1]
\le  \mathbb E_{x \in U^n} [(r - 1)^2 \cdot|f(x) - g(x)|^2] \\
&\le& (r - 1)^2 \cdot \mathbb E_{x \in U^n} [|f(x) - g(x)|^2]
\leq \epsilon (r - 1)^2.
\end{eqnarray*}
The last inequality follows from the definition of $\eps$-approximation.
\end{proofof}

Extension of our learning algorithm to the agnostic setting follows from the result of Gopalan, Kalai and Klivans.

\begin{theorem}[\cite{GKK08}]
If every function $f$ in a class $C$ has an $M$-sparse $\epsilon$-approximation, then
there is an agnostic learning algorithm for $C$ with running time $poly(n,M,1/\epsilon)$.
\end{theorem}

\end{proofof}


\subsubsection*{Acknowledgments}
We are grateful to Jan Vondrak, Vitaly Feldman, Lev Reyzin, Nick Harvey, Paul Beame, Ryan O'Donnell and other people for their feedback and comments on the results in this paper.

\bibliography{../submodular}

\appendix
\section{Converting a learner into a proper learner}
Let $\cal C$ be a class of discrete objects represented by functions over a domain of ``size'' $n$.
\begin{proposition}\label{prop:lerner-to-proper-learner}
If there exists a learning algorithm $L$ for a class $\mathcal C$ with query complexity $q(n,\epsilon)$ and running time $t(n, \epsilon)$, then there exists a proper learning algorithm $L'$ for $\mathcal C$ with query complexity $q(n, \epsilon / 2)$ and running time $t(n, \epsilon / 2) + |{\mathcal C}|$.
\end{proposition}
\begin{proof}
Given parameters $n, \epsilon$ and oracle access to a function $f$, the algorithm $L'$ first runs $L$ with parameters $n, \epsilon / 2$ to obtain a hypothesis $g$.
Then it finds and outputs a function $h \in {\mathcal C}$, which is closest to $g$, namely $h = argmin_{h' \in \mathcal C} dist(g,h')$.
By our assumption that $L$ is a learning algorithm,  $dist(f,g) \le \epsilon / 2$.
Since $f \in {\mathcal C}$, we have $dist(g,h) \le dist(g,f) \le \epsilon / 2$.
By the triangle inequality, $dist(f,h) \le dist(f,g) + dist(g,h) \le \epsilon$.
\end{proof}


\section{Fourier analysis}\label{app:fourier-analysis}\snote{Organize appendices; this lemma should appear before its proof.}
\subsection{Proof of Lemma~\ref{lem:l1-expectation-bound}}\label{sec:proof-of-l1-expectation-bound}
\begin{proofof}{Lemma~\ref{lem:l1-expectation-bound}}
Consider a random variable $\mathcal L$ supported on $\domain$, such that for each $x_i$, independently $\Pr[x_i \in \mathcal L] = p$.
The random variable $\mathcal L$ is the set of live variables in a random restriction with parameter $p$.
We can rewrite $L_{1,\myindexk}$ as:
\begin{align}
L_{1,\myindexk}(f) = \sum_{|S| = \myindexk} |\hat f(S)| =
\left(\frac{1}{p}\right)^\myindexk \mathbb E_{\mathcal L} \left[\sum\limits_{S \subseteq \mathcal L, |S| = \myindexk} \left| \hat f(S)\right| \right]. \label{eqn:l1-averaging}
\end{align}

For an arbitrary choice of $\mathcal L$ and a subset $S \subseteq \mathcal L$ we have:
\begin{align*}
|\hat f(S)| &= \left|\mathbb E_{x_1, \dots, x_\mysize} \left[f(x_1, \dots, x_\mysize) \chi_S(x_1, \dots, x_\mysize)\right]\right| \\
&\le \mathbb E_{x \notin \mathcal L} |\mathbb E_{x \in \mathcal L} \left[f(x_1, \dots, x_n) \chi_S(x_1, \dots, x_\mysize)\right]| \\
&= \mathbb E_{\rho} \left[|\hat{\myrand{f}}(S)| \mid live(\rho) = \mathcal L \right], \\
\end{align*}
where the last line follows from the observation that averaging over $x_i \notin \mathcal L$ is the same as taking the expectation of a random restriction whose set of live variables is restricted to be $\mathcal L$.
Because the absolute value of every coefficient $S$ is expected to increase, this implies that\snote{Hat position in eqn below}:
$$\sum_{S \subseteq \mathcal L} \left|\hat f(S)\right| \le \mathbb E_{\rho} \left[\sum_{S \subseteq \mathcal L, |S| = \myindexk} |\hat{\myrand{f}(S)}| \mid live(\rho) = \mathcal L \mid\right] = \mathbb E_{\rho} \left[L_{1,\myindexk}(f_\rho) | live(\rho) = \mathcal L \right].$$
Using this together with~(\ref{eqn:l1-averaging}) we conclude that:
$$L_{1,\myindexk}(f) = \left(\frac{1}{p}\right)^\myindexk \mathbb E_{\mathcal L} \left[\sum\limits_{S \subseteq \mathcal L, |S| = \myindexk} \left| \hat f(S)\right| \right] \le \left(\frac{1}{p}\right)^\myindexk \mathbb E_{\rho}\left[L_{1,\myindexk}(\myrand{f})\right]$$

\end{proofof}

\subsection{Proof of Lemma~\ref{lem:l2-bound}}
\begin{proofof}{Lemma~\ref{lem:l2-bound}}
\begin{lemma}[\cite{M95,O12}]
\snote{Quantify $t$.}Let $f \colon \{0,1\}^n \rightarrow \{-1,1\}$ and $f_{\rho}$ be a random restriction with parameter $p$. Then
$$\sum_{|S| > t} \hat f^2 (S) \le \Pr_{\rho}[\fdeg(\myrand{f}) \ge tp / 2].$$
\end{lemma}
Because $\fdeg(\myrand{f}) \le \dtd(\myrand{f})$ \gnote{List as a basic fact}\snote{I agree.}
and thus $\Pr[\fdeg(\myrand{f}) \geq tp / 2] \le \Pr[\dtd(\myrand{f}) \ge tp / 2]$.
By using Lemma~\ref{lem:gen-switching-lemma} and setting $p = 1/14k$ and $t = 28k \log (2r /\epsilon)$, we complete the proof.
\end{proofof}

\subsection{Proof of Lemma~\ref{lem:l1-bound}}
\begin{proofof}{Lemma~\ref{lem:l1-bound}}
Let $L_{1,\myindexk}(f) = \sum_{|S| = \myindexk} |\hat f(S)|$
and $L_1(f) = \sum_{\myindexk = 0}^n L_{1,\myindexk}(f) =  \sum_{S} |\hat f(S)|$.

We use the following bound on $L_1(f)$ for decision trees.
\begin{proposition}[\cite{KM93,O12}]\label{prop:l1-decision-trees}
Consider a function $f$\snote{domain/range?}, such that $\dtd(f) \le s$. Then $L_1(f) \le 2^s$.
\end{proposition}
We show the following generalization of Lemma 5.2 in~\cite{M95} for $\mydnf$.
\begin{lemma}\label{lem:l1-random-restriction}
Let $f \in \mydnf$ and let $\rho$ be a random restriction of $f$ with parameter $p \le 1 / {28 \mywidth}.$ Then $\mathbb E_{\rho} \left[L_1(\myrand{f})\right] \le 2 r.$
\end{lemma}
\begin{proofof}{Lemma~\ref{lem:l1-random-restriction}}
By the definition of expectation,
$$\mathbb E_{\rho}[L_1(f)] =
\sum_{s = 0}^n \Pr[\dtd{\myrand{f}} = s] \cdot \mathbb E_{\rho}\left[L_1(\myrand{f}) \mid \dtd(\myrand{f}) = s\right].$$
By Proposition~\ref{prop:l1-decision-trees} for all $\rho$, such that $\dtd(\myrand{f}) = s$, it holds that $L_1(f) \le 2^s$.
By Lemma~\ref{lem:gen-switching-lemma} we have $\Pr[\dtd(\myrand{f}) \ge s] \le r (7p\mywidth)^s$. Therefore $\mathbb E_{\rho}[L_1(f)] \le \sum_{s = 0}^n r(7p\mywidth)^s 2^s = r \cdot \sum_{s= 0}^n (14p\mywidth)^s$.
For $p \le 1/{28 \mywidth}$ the lemma follows.
\end{proofof}
We use Lemma 5.3 from~\cite{M95} to bound $L_{1,\myindexk}(f)$ by the value of $\mathbb E_{\rho}\left[L_{1,\myindexk}(\myrand{f})\right]$. Because in~\cite{M95} the lemma is stated  for Boolean functions, we give the proof for real-valued functions in Appendix~\ref{sec:proof-of-l1-expectation-bound} for completeness.
\begin{lemma}[\cite{M95}]\label{lem:l1-expectation-bound}
For $f \colon \{0,1\}^n \rightarrow [-1,1]$ and a random restriction $\rho$ with parameter $p$,
$$L_{1,\myindexk}(f) \le \left(\frac{1}{p}\right)^\myindexk \mathbb E_{\rho}\left[L_{1,\myindexk}(f)\right].$$
\end{lemma}

Note that $\sum_{S \colon |S| \le \tau} |\hat f(S)| = \sum_{\myindexk = 0}^{\tau} L_{1,\myindexk}(f)$. By setting $p = 1/{28 k}$ and using Lemmas~\ref{lem:l1-random-restriction} and~\ref{lem:l1-expectation-bound}, we get:
$$L_{1,\myindexk}(f) \le 2r (28k)^t.$$
Thus, $\sum_{S \colon |S| \le \tau} |\hat f(S)| \le 4 r (28\mywidth)^{\tau} = r k^{O(\tau)}$, completing the proof.
\end{proofof}

\section{Omitted proofs from Section~\ref{sec:submodular-to-DNF} and Section~\ref{sec:switching-lemma}}\label{app:omitted-proofs}
\subsection{Proof of Proposition~\ref{prop:monotone-cover}}
\begin{proofof}{Proposition~\ref{prop:monotone-cover}}
The proof is by induction on the value $f([n])$ that the function $f$ takes on the largest set in its domain.
The base case of induction is $f([n]) = k$.
In this case, $\mathcal S$ consists of a single set $S = [n]$, and the function $f$ is monotone non-increasing on $\down{S} = \downdec{S}$.
Now suppose that the statement holds for all $f$, such that $f([n]) \ge t$.
If $f([n]) = t - 1$ then for all $Y$, such that there exists a set $Z$ of size $n - 1$ such that $f(Z) > f([n])$ and $Y \in \down{Z}$ there exists a set $S \in \mathcal{S}$, such that $Y \in \down{S}$ by applying inductive hypothesis to $f_{\down{Z}}$\snote{The sentence is too long, can't parse.}.
Otherwise, $Y \in \downdec{[n]}$, completing the proof.
\end{proofof}

\subsection{Proof of Claim~\ref{clm:switching-lemma-computation}}
\begin{proofof}{Claim~\ref{clm:switching-lemma-computation}}
We have $|\mathcal R^{\ell}_n| = \binom{n}{\ell} 2^{n - \ell}$, so:
$$\frac{|\mathcal R^{\ell - s}_n|}{|\mathcal R^{\ell}_n|} \le \frac{(2\ell)^s}{(n - \ell)^s}.$$
We use the following bound on $|\stars(k,s)|$.

\begin{proposition}[Lemma 2 in~\cite{B94}]\label{prop:stars-bound}
$|\stars(k,s)| <(k /\ln 2)^s.$
\end{proposition}
Using Proposition~\ref{prop:stars-bound} we get:
\begin{align*}
\frac{|S|}{|\mathcal R^{\ell}_n|} &\le \frac{|\mathcal R^{\ell - s}_{n}|}{|\mathcal R^{\ell}_n|} \cdot |\stars(k,s)| \cdot 2^s \\
&\le \left(\frac{4 \ell k}{(n - \ell) \ln 2}\right)^s \\
&= \left(\frac{4pk}{(1 - p) \ln 2}\right)^s.
\end{align*}
For $p < 1/7$, the last expression is at most $(7pk)^s$, as claimed.
\end{proofof}
\end{document}